\documentclass[11pt]{article}
\def\anonymous{0}
\def\comments{1}
\ifnum\anonymous=1
    \def\comments{0}
\fi
\usepackage[margin=1in]{geometry}
\usepackage{amsmath,amssymb,amsfonts,amsthm,epsfig}
\usepackage[usenames,dvipsnames]{xcolor}
\usepackage{bm,xspace}
\usepackage{tcolorbox}
\usepackage{cancel}
\usepackage{fullpage}
\usepackage{liyang}
\usepackage{framed}
\usepackage{verbatim}
\usepackage{enumitem}
\usepackage{array}
\usepackage{multirow}
\usepackage{afterpage}
\usepackage{mathrsfs}
\usepackage{pifont} 
\usepackage{chngpage}
\usepackage[normalem]{ulem}
\usepackage{boxedminipage}
\usepackage{caption}
\usepackage{forest}
\usepackage{microtype} 

\usepackage{algorithm}
\usepackage{algorithmicx}
\usepackage{algpseudocode}

\usepackage{pgfplots}
\pgfplotsset{width=8cm,compat=newest}

\ifnum\comments=1
    \newcommand{\lnote}[1]{\footnote{{\bf \color{blue}Li-Yang}: {#1}}}
    \newcommand{\gnote}[1]{\footnote{{\bf \color{violet}Guy}: {#1}}}
    \newcommand{\knote}[1]{\footnote{{\bf \color{orange}Konstantina}: {#1}}}
    \newcommand{\lynote}[1]{\footnote{{\bf \color{magenta}Lydia}: {#1}}}
    \newcommand{\jnote}[1]{\footnote{{\bf \color{red}Jon}: {#1}}}
\else 
    \newcommand{\lnote}[1]{}
    \newcommand{\gnote}[1]{}
    \newcommand{\knote}[1]{}
    \newcommand{\lynote}[1]{}
    \newcommand{\jnote}[1]{}
\fi
\def\colorful{1}

\ifnum\colorful=1

\fi
\ifnum\colorful=0

\fi

\newcommand{\error}{\mathrm{error}}
\newcommand{\avgerror}{\mathrm{avg}\text{-}\mathrm{error}}

\newcommand{\proj}{\mathrm{proj}}

\newcommand{\VC}{\mathrm{VC}}
\newcommand{\gen}{\mathrm{gen}}

\DeclareMathOperator*{\argmax}{arg\,max}

\newlist{enumprop}{enumerate}{1} 
\setlist[enumprop]{label=\arabic*.,ref=\theproposition.\arabic*}
\crefalias{enumpropi}{proposition}

\makeatletter
\newtheorem*{rep@theorem}{\rep@title}
\newcommand{\newreptheorem}[2]{
\newenvironment{rep#1}[1]{
 \def\rep@title{#2 \ref{##1}}
 \begin{rep@theorem}\itshape}
 {\end{rep@theorem}}}
\makeatother

\newreptheorem{theorem}{Theorem}

\begin{document}

\title{Multitask Learning via Shared Features: \\ Algorithms and Hardness}
\ifnum\anonymous=1
\author{Anonymous Author(s)}
\date{}
\else
\author{
    Konstantina Bairaktari\thanks{Khoury College of Computer Sciences, Northeastern University. Supported by NSF grants CCF-1750640, CNS-1816028, CNS-2120603, CCF-1909314, and CCF-1750716. \texttt{bairaktari.k@northeastern.edu}}\hspace{10mm}
    \and Guy Blanc\thanks{Department of Computer Science, Stanford University. Supported by NSF grant CCF-1942123.}\hspace{10mm}
    \and Li-Yang Tan\thanks{Department of Computer Science, Stanford University. Supported by NSF grant CCF-1942123.}
    \and Jonathan Ullman\thanks{Khoury College of Computer Sciences, Northeastern University.  Supported by NSF grants CCF-1750640, CNS-1816028, and CNS-2120603. \texttt{jullman@ccs.neu.edu}}\hspace{10mm}
    \and Lydia Zakynthinou\thanks{Khoury College of Computer Sciences, Northeastern University. Supported by NSF grants CCF-1750640, CNS-1816028, CNS-2120603, and a Facebook PhD Fellowship. \texttt{zakynthinou.l@northeastern.edu}}
}
\date{\small{\today}}
\fi

\maketitle

\begin{abstract}
We investigate the computational efficiency of multitask learning of Boolean functions over the $d$-dimensional hypercube, that are related by means of a feature representation of size $k\ll d$ shared across all tasks. We present a polynomial time multitask learning algorithm for the concept class of halfspaces with margin $\gamma$, which is based on a simultaneous boosting technique and requires only $\mathrm{poly}(k/\gamma)$ samples-per-task and $\mathrm{poly}(k\log(d)/\gamma)$ samples in total. 

In addition, we prove a computational separation, showing that assuming there exists a concept class that cannot be learned in the attribute-efficient model, we can construct another concept class such that can be learned in the attribute-efficient model, but cannot be multitask learned efficiently---multitask learning this concept class either requires super-polynomial time complexity or a much larger total number of samples.

\end{abstract}

\section{Introduction}

A remarkable pattern in modern machine learning is that complex models often transfer surprisingly well to solve new tasks with very little additional data for that task---far less than one would need to solve that task from scratch.  This sort of \emph{multitask learning} (sometimes called \emph{meta learning}, \emph{transfer learning}, or \emph{few-shot learning})~\cite{Thr98,TM95,AEP08,PM13,Bax97,TP12,SSSG17} is possible because the tasks share some common structure that makes a model for one task relevant for solving the others.  One useful structure is a \emph{shared representation} common to all tasks.  We assume that there is a low-dimensional representation of the data that is sufficient for solving every learning task---because the relevant features are shared across all tasks, we can pool the data for all tasks to find the representation, and because the representation is low-dimensional, each task can be solved with relatively few samples.  

The existence of low-dimensional representations often provably reduces the number of samples needed to solve each task~\cite{MaurerPR16,TripuraneniJJ20}, and popular heuristics like MAML~\cite{FinnAL17} or other gradient-based methods (e.g.~\cite{NicholS18,RaghuRBV19,AntoniouES19}) are often successful at multitask learning.  However, much less is known about \emph{computationally efficient algorithms} for exploiting these shared representations, or about complexity-theoretic barriers that are specific to multitask learning. While there is an elegant emerging body of research on provably efficient methods for multitask learning~\cite{BalcanBV15, DuHKLL20, TripuraneniJJ21, ThekumparampilJNO21, CollinsMOS22}, so far this work is currently limited to simple regression problems, or makes strong distribution assumptions, or both, so little is known about algorithms for classification tasks or for more general distributions.

In this work we focus on the particular setting of binary classification where the shared representation is simply a subset of the input features.  In this setting we give the first computationally efficient algorithm for multitask learning of halfspaces, under a distribution-free margin assumption on the halfspaces but with no distributional assumptions.  This result follows from a more general extension of \emph{AdaBoost}~\cite{FreundS97} to the multitask setting.  We also prove a computational separation, showing that under a natural complexity assumption, there is a concept class such that: (1) The class can be learned in the \emph{attribute-efficient model}~\cite{Littlestone87}---any single task can be learned in polynomial time with sample complexity proportional to the number of relevant features.  (2) The class cannot be multitask learned efficiently---the corresponding multitask problem where we want to learn multiple concepts over the same set of features cannot be solved efficiently unless the total data across all tasks is much larger.

\subsection{Our Results}

We now give a more detailed, but still high-level and informal statement of our two main results. To describe our results we need to introduce some notation, although we defer formal preliminaries to~\Cref{sec:prelims}.  We assume that there are $n$ learning tasks.  For each task there is a corresponding distribution $D^{(i)}$ over labeled examples $(x,y) \in \{\pm 1\}^d \times \{\pm 1\}$. We assume realizable tasks so that there is a known \emph{concept class} $\mcC$ and for each task there is some $f^{(i)} \in \mcC$ such that
$$
\Prx_{(x,y) \sim D^{(i)}}[f^{(i)}(x) = y] = 1
$$
Our learning algorithm is given $m$ samples from each distribution $D^{(i)}$, for a total of $mn$ samples, and must return $\hat{f}^{(1)},\dots,\hat{f}^{(n)}$ that label the data well on average over the tasks
$$
\frac{1}{n} \sum_{i=1}^{n} \Prx_{(x,y) \sim D^{(i)}}[\hat{f}^{(i)}(x) = y] \approx 1
$$
A na\"ive baseline solution is to solve each task separately, and our goal is to improve over this baseline.  Since we can't hope to do so without some relationship between the functions $f^{(i)}$, in this work we consider cases where:
\begin{enumerate}
    \item the functions in $\mcC$ have a small number of \emph{relevant coordinates}, and 
    \item the total number of relevant coordinates among all of the functions $f^{(1)},\dots,f^{(n)}$ is small.
\end{enumerate}
The set of up to $k$ coordinates in total that are relevant for at least one of the tasks is what we call the \emph{shared representation}---they are shared because they are the same across tasks, and they are a representation because we can think of these coordinates as a way of reducing examples to a lower-dimensional form that suffices for learning.

\medskip\textbf{Efficient Multitask Learning of Halfspaces.}
First we consider the case where $\mcC$ consists of \emph{$k$-sparse, $\gamma$-margin halfspaces}, meaning functions of the form
$
f(x) = \mathrm{sign}(\theta \cdot x)
$
for some $\theta \in \R^d$ with at most $k$ non-zero coordinates, which also satisfy the condition
$$
\forall i \in [n] \ \forall x \in \{\pm 1\}^d \quad |\theta \cdot x| \geq \gamma \cdot \|\theta\|_2
$$
Note that our margin condition is quite different from, and incomparable to, the standard margin conditions used in machine learning, both in how it normalizes the data and because it's a condition that only depends on the function and not the distribution-function pair.\footnote{The typical margin condition used in the literature requires that $\Pr_{x \sim D}[|\theta \cdot x| \geq \gamma \| \theta \|_2 \| x \|_2]=1$.  For example this condition is used to analyze the classical Perceptron algorithm for learning halfspaces.}

Since each halfspace is $k$-sparse, the na\"ive baseline of solving each task separately would require $m = O(k \log(d))$ samples for each task.  In many applications, the number of total features $d$ may be extremely large relative to the number of relevant features $k$, so we think of drawing $\log(d)$ samples-per-task as prohibitive.  Our main theorem says that, assuming a shared representation of $k$ features, we can indeed do better provided the number of tasks is at least $\log(d)$.
\begin{theorem}[Informal] \label{thm:intro-halfspace}
    Suppose we have $n$ distributions $D^{(1)},\dots,D^{(n)}$ each labeled by a $k$-sparse, $\gamma$-margin halfspace. There is a $\mathrm{poly}(d,k,n,m,1/\gamma)$-time algorithm that draws $m$ samples from each task and solves the multitask learning problem provided each task has $m \geq \mathrm{poly}(k/\gamma)$ samples and there are $nm \geq \mathrm{poly}(k \log(d) / \gamma)$ samples in total.
\end{theorem}
The key advantage of our result is that the number of samples-per-task can be entirely independent of $d$ provided we have enough tasks to solve.

We achieve this result via a generalization of the AdaBoost algorithm for the multitask setting.  In the single-task setting, a boosting algorithm takes a sequence of \emph{weak learners} that predict the label slightly better than random, and combines them to obtain a \emph{strong learner} that predicts the labels nearly perfectly.  We give a variant for the multitask setting that takes a sequence of weak learners that have some advantage over random on average over tasks and combine them to obtain a new learner that predicts labels nearly perfectly.

\medskip\textbf{Separating Multitask and Attribute-Efficient Learning.} One way to interpret~\Cref{thm:intro-halfspace} is that it shows that multitask learning of sparse, large-margin halfspaces, is not that much harder than \emph{attribute-efficient learning} of the same class.  In the attribute-efficient learning model~\cite{Littlestone87}, we have a learning problem on examples in $\{\pm 1\}^d$ where functions in $\mcC$ have at most $k \ll d$ relevant variables, and our goal is to learn in polynomial time, with sample complexity of the form $s = f(k) \cdot \mathrm{poly}(\log(d))$.  Multitask learning strictly generalizes this model by splitting the samples in $n$ tasks with $m = s/n$ samples-per-task with each task labeled by a different function.

Attribute-efficient learning is a very challenging problem and large-margin halfspaces are one of the few classes for which we do know an attribute-efficient learning algorithm~\cite{Littlestone87,Val99}.  Perhaps we can efficiently multitask learn \emph{any} concept class that is learnable in the attribute-efficient model?  We give strong evidence that this is not the case, by showing that, under a plausible assumption, there is some class $\mcC'$ consisting of functions of $k+\log\log(d)$
relevant variables, that can be learned in the attribute-efficient model but for which there is no efficient multitask learning algorithm.
\begin{theorem}[Informal]\label{thm:intro-hardness}
    Assume there is some class of functions $\mcC$ with $k$ relevant variables and a distribution $D_{x}$ over unlabeled examples such that $\mcC$ is learnable over $D_{x}$, but every efficient learner for $\mcC$ requires many more samples.  Then there is another concept class $\mcC'$ over $d' \lesssim d+\log\log(d)$ features with $k' \lesssim k + \log\log(d)$ relevant variables and a distribution $E$ over the remaining features, such that:
    \begin{enumerate}
        \item $\mcC'$ can be attribute-efficient learned over $D_{x}\times E$,
        \item $\mcC'$ can be multitask learned with few samples in exponential time, but
        \item any polynomial-time algorithm that multitask learns $\mcC'$ requires many more samples.
    \end{enumerate}
\end{theorem}
An interesting feature of this result is that it gives \emph{representation-independent} hardness, meaning it makes no assumptions about the form of the multitask learner's output.

We note that there are many concept classes that could be used to instantiate our assumption, an obvious example being $k$-sparse parities.  For parities, $O(k \log(d))$ samples suffice for exponential-time algorithms, but the best known efficient algorithms require $\Omega(d^{1-1/k})$ samples~\cite{KlivansS06}.   Our result is quantitative and can be instantiated with many choices of parameters, so we give a concrete example.  If we assume that the current attribute-efficient learning algorithms for $k$-sparse parities are optimal, then we get another concept class that can be attribute-efficient learned with $s = \mathrm{poly}(k\log(d))$ samples, and can be multitask learned with $nm = O(k \log(d))$ total samples and $m = O(k\log\log(d))$ samples-per-task, but for which any polynomial-time multitask learner with $n = \Omega(\log d)$ tasks requires $nm = \Omega(d^{1-1/k})$ total samples.

We note that our separation only holds for attribute-efficient and multitask learning for the specific distribution $E$ over the $d'-d$ new features of class $\mcC'$ of our example.  In contrast, our positive result for learning halfspaces is distribution-free.  It is an intriguing open problem to separate \emph{distribution-free} attribute-efficient and multitask learning.


\section{Preliminaries} \label{sec:prelims}
\newcommand{\rel}{\mathrm{Rel}}
For an integer $n$, we write $[n]=\{1,\ldots, n\}$. Each unlabelled sample is a vector of features $x\in \bits^d$. We use $x_j$ to denote the $j$-th feature of sample $x$ and $x_{-j}$ to denote the vector $x$ with the $j$-th coordinate removed. 
We consider Boolean classifiers of the form $f:\bits^d\to\bits$. We also write $v_{|v_j=b}$ to denote the vector $v$ where the $j$-th coordinate is set to $b\in\bits$. We say a feature is \emph{relevant} if changing its value has the potential to change the value of the function, and define
\[
\rel(f)=\{j\in[d] \mid \exists v_{-j} \in \bits^{d-1} \text{ s.t. } f(v_{|v_j=+1})\neq f(v_{|v_j=-1})\}.
\]

We write $(x,y)\sim D$ when $(x,y)$ is drawn from a distribution $D$. We denote the support of a distribution $D$ by $\mathrm{Supp}(D)$.

 In the multitask learning setting, we assume that there exist $n$ tasks (or users) and each task $i\in[n]$ consists of a distribution $D^{(i)}$ over labelled samples in $\bits^d\times \bits$ and a classification function $f^{(i)}$ such that $y=f^{(i)}(x)$ for any $(x,y)\in\mathrm{Supp}(D)$. For each task, we receive a sample set of size $m$, 
 denoted by
 \[
 S^{(i)}= \left\{(x^{(i)}_1, y^{(i)}_1), \ldots, (x^{(i)}_m, y^{(i)}_m)\right\}
 \]
 where $(x^{(i)}_j, y^{(i)}_j)$ is the $j$-th sample drawn i.i.d.\ from $D^{(i)}$ and $y^{(i)}_j=f^{(i)}(x^{(i)}_j)$ is its label.
 Our goal is to design a learning algorithm, which, given $S^{(1)},\ldots, S^{(n)}$, returns \emph{hypotheses} for each task, $h^{(1)}, \ldots, h^{(n)}$, with small average error. 
 For each task $i\in[n]$ and hypothesis $h^{(i)}$, we define the population and training error by
\begin{equation*}
\error^{(i)}(h^{(i)}) \coloneqq \Prx_{(x,y) \sim D^{(i)}}[h^{(i)}(x) \neq y] 
\text{~and~} \widehat{\error}^{(i)}(h^{(i)})=\frac{1}{m}\sum_{(x,y)\in S^{(i)}} \Ind[h^{(i)}(x)\neq y],
\end{equation*}
respectively and the population and training average error by
\begin{equation*}
\avgerror(h^{(1)},\ldots, h^{(n)}) \coloneqq \frac{1}{n} \sum_{i \in [n]} \error^{(i)}(h^{(i)})
\text{~and~}\widehat{\avgerror}(h^{(1)},\ldots, h^{(n)}) \coloneqq \frac{1}{n} \sum_{i \in [n]} \widehat{\error}^{(i)}(h^{(i)}),
\end{equation*}
respectively.
More formally, the single task setting, where $n=1$, is equivalent to the PAC (\emph{probably approximately correct}) learning model (in its realizable case), introduced by \cite{Valiant84}, where our goal is to return a single hypothesis $h^{(1)}$ so that with probability at least $1-\delta$ over the randomness of the dataset and the algorithm, $\error^{(1)}(h^{(1)})\leq\varepsilon$, for given accuracy parameters $\varepsilon,\delta \in (0,1)$. Similarly, in the multitask setting, our goal is to return hypotheses $h^{(1)}, \ldots, h^{(n)}$ so that with probability at least $1-\delta$, $\avgerror(h^{(1)},\ldots, h^{(n)})\leq \varepsilon$.


Multitask learning can be more sample-efficient overall than single-task learning when the tasks are related. Intuitively, if samples for one task are informative for learning a good hypothesis for another task, then pooling all samples to learn all tasks simultaneously may require less samples in total than learning each task separately. In order to formalize the relationship between tasks, we adopt one of the standard assumptions in the literature, that of a \emph{shared-feature representation}. More specifically, this representation will take the form of a small subset of the features that contains the relevant variables for every task (although we can consider other constraints on the set of relevant features).

\begin{definition}[Multitask Learning under Shared-Feature Representation]\label{def:mtl}
Let $\mcC$ be a class of functions $f:\bits^d\to\bits$ and let $\mcV\subseteq 2^{[d]}$ be a collection of subsets of $[d]$. 
We say that $\mcC$ is \emph{$\mcV$-multitask learnable for a class of distributions $\mcD$ with $n$ tasks, $m$ samples-per-task, and accuracy parameters $\varepsilon,\delta\in(0,1)$} 
if there exists algorithm $\mcA$ such that $\forall f^{(1)}, \ldots, f^{(n)}\in\mcC$ that satisfy $\bigcup_{i\in[n]} \rel(f^{(i)})\in \mcV$, 
$\forall D^{(1)},\ldots, D^{(n)} \in \mcD$, 
given $m$ i.i.d.\ samples from each $D^{(i)}$ labeled by $f^{(i)}$, 
returns hypotheses $h^{(1)}, \ldots, h^{(n)}$ such that with probability at least $1-\delta$ over the randomness of the samples and the algorithm, $\avgerror(h^{(1)}, \ldots, h^{(n)}) \leq \varepsilon$, i.e.,
\[\frac{1}{n}\sum_{i\in [n]}\Prx_{(x,y)\sim D^{(i)}}[h^{(i)}(x)\neq y]\leq \varepsilon.\]
\end{definition}

Note that, under this definition, it is not required that all features in $V$ are relevant for all tasks but rather that for each task, the relevant features belong in $V$.




Instead of minimizing the average error among all $n$ tasks, another natural but stronger requirement would be to minimize the maximum error per task: $\max_{i \in [n]} \error^{(i)}$. This cannot be achieved under the shared-feature representation assumption in the general case. 
Suppose that $n-1$ tasks only depend on a single variable whereas the first task depends on the remaining $k-1$ relevant variables. In this case, only the samples of $S^{(1)}$ are informative for the first task and so returning a hypothesis $h^{(1)}$ with $\error^{(1)}(h^{(1)})\leq \eps$ would require $m$ to be as large as required for the single-task setting (for example, for $(k-1)$-sparse halfspaces, $m=O(k\log(d))$). This is in contrast to other settings (e.g. the collaborative learning setting~\cite{BlumHPQ17}) which however make much stronger assumptions on the relationship between tasks.

\paragraph{Comparison with attribute-efficient learning}
We give a formal definition of attribute-efficient learning~\cite{Littlestone87} here. In~\Cref{sec:lower}, we will construct a class of functions and distribution for which attribute-efficient learning is feasible in polynomial time but multitask learning is not, unless the total number of samples is much larger.
\begin{definition}[Attribute-Efficient Learning]\label{def:ael}
Let $\mcC$ be a class of functions $f:\bits^d\to\bits$. 
Let $\mathrm{len}(\mcC)$ denote the description length of the class.\footnote{We will use a binary encoding scheme.}
We say that $\mcC$ is \emph{attribute-efficient learnable for a class of distributions $\mcD$ with accuracy parameters $\varepsilon,\delta\in(0,1)$}, if there exists a $\mathrm{poly}(d)$-time algorithm $\mcA$ such that $\forall f\in\mcC$, $\forall D\in\mcD$, given $s=\mathrm{poly}(\mathrm{len}(\mcC))$ samples from $D$ labeled by $f$, returns a hypothesis $h$ such that with probability at most $1-\delta$ over the randomness of the sample and the algorithm, $\Prx_{(x,y)\sim D} [h(x)\neq y]\leq \varepsilon$.
\end{definition}
For example, in the case where the concept class $\mcC$ is the class of parity functions on at most $k$ features, the size of the class is $|C|=O\left(\binom{d}{k}\cdot 2^k\right)$ and its description length is $\mathrm{len}(\mcC)=\log|\mcC|=O(k\log(d))$. 
So in this case, attribute-efficient learning $\mcC$ would require a learning algorithm with sample complexity $s = \mathrm{poly}(k\log(d))$. 
As mentioned above, multitask learning with $n$ tasks and $m = s/n$ samples-per-task is strictly more general and is reduced to the attribute-efficient setting  when all tasks have the same classifier and distribution over unlabelled examples. In this example, we would be interested in designing a multitask learning algorithm with $n$ tasks, each having a (potentially different) parity function over a subset of the shared features $V\subseteq [d]$, $|V|\leq k$, and a distribution $D^{(i)}$ over $\bits^d\times \bits$, with $m=\mathrm{poly}(k\log(d))/n$ samples-per-task. What is more, since we consider the total number of features $d$ to be too high, we would be interested in multitask learning algorithms which are accurate even in the regime where $m\ll \log(d) <s$.

\section{Efficient multitask learning of halfspaces}
In this section, we present a multitask learning algorithm for the case where each task is a large-margin halfspace classifier over some subset of a common set of features $V\subseteq [d]$ such that $|V|\leq k$. For each task $i\in[n]$, we have a function $f^{(i)}:\bits^d \to \bits$ of the form $f^{(i)}(x) = \sign(\theta^{(i)}\cdot x)$, where $\theta^{(i)}\in\R^d$. For simplicity, we only consider linear separators through the origin.
Furthermore, we assume that each task's classifier $f^{(i)}$ is such that no example $x\in\bits^d$ falls too close to the boundary of the halfspace $\theta^{(i)}$. That is, we assume that all classifiers $f^{(i)}$ are \emph{halfspaces with margin $\gamma$}
, as defined below.

\begin{definition}\label{def:margin}
    Let $f:\bits^d \to \bits$ be a classifier of the form $f(x) = \sign(\theta\cdot x)$. 
    We say that such an $f$ is \emph{a halfspace with margin $\gamma$}
    if it holds that $\forall x\in \bits^d$, $$\frac{|\theta\cdot x|}{\|\theta\|_2}\geq \gamma.$$ We call $\gamma$ the \emph{margin}.
    
\end{definition}
Note that, \emph{for a fixed $x\in\bits^d$}, this is a weaker condition than the standard large-margin assumption that is used in single-task learning. The standard large-margin assumption requires that $\frac{|\theta\cdot x|}{\|\theta\|_2\|x\|_2}\geq \gamma$ for all $x\in\mathrm{Supp}(D)$. In our setting $\|x\|_2=\sqrt{d}$ and thus our assumption, for fixed $x\in\bits^d$, is weaker. 
In fact, under the standard assumption, a single task can already be solved using a sample size of $\tilde{O}(1/\gamma^2\eps)$ (see \cite{Blum06} and references therein), which is independent of $d$, yet if we executed the same analysis under our assumption the bound would translate to $\tilde{O}(d/\gamma^2\eps)$. 
In general, the two assumptions are incomparable, since the standard large-margin assumption depends on the distribution $D$ and asks that there exists a margin $\forall x\in \mathrm{Supp}(D)$, whereas ours does so $\forall x\in\bits^d$.

The main theorem of this section is the following.
\begin{theorem}[Large-margin halfspaces are multitask learnable]\label{th:upperbound}
Let $\mcV_k=\{V\subseteq [d] \mid |V|\leq k\}$. Let $\mcC_{\gamma}$ be the class of halfspaces with margin $\gamma$ (\Cref{def:margin}). For any $\eps,\delta \in (0,1)$,
\begin{align*}
    m &= \Omega\left(\frac{k^2\log^2(1/\eps)}{\gamma^2\eps} \right),\\
    nm &= \Omega\left(\frac{k^2 \log (d)\log(1/\eps)}{\gamma^2\eps} + \frac{\log(1/\delta)}{\eps}\right).
\end{align*}
Class $\mcC_{\gamma}$ is $\mcV_k$-multitask learnable with $n$ tasks, $m$ samples-per-task, and accuracy parameters $\varepsilon,\delta\in(0,1)$, in time $O(\frac{nmdk^2\log(1/\eps)}{\gamma^2})$.
\end{theorem}

In particular, as long as $n = \Omega(\log d)$ and $\delta = d^{-O(1)}$, having $m =  O\left(\frac{k^2\log^2(1/\eps)}{\gamma^2\eps} \right)$ samples-per-task suffices. To prove~\Cref{th:upperbound}, we need the following fact that says that large-margin classifiers have a feature that is highly correlated with the label. This feature will serve as a weak learner in our analysis.

\begin{fact}[Discriminator Lemma]
    \label{fact:disc-lemma}
    Let $f:\bits^d\to \bits$ be a halfspace classifier over features in $V$, such that $|V|\leq k$. That is, $f(x) = \sign(\theta\cdot x)$ and $\theta_j=0$ for all $j\notin V$. For every distribution $D$ over samples, there exists a feature $\ell\in V$ such that $$|\Ex_{(x,y)\sim D}[f(x)\cdot x_{\ell}]|\geq \frac{\gamma}{\sqrt{k}}.$$
\end{fact}
There exist several similar statements in the literature, see for example~\cite{HajnalMPTS93}. We prove this version here for completeness.
\begin{proof}
By assumption, $\forall x\in \bits^d$, $\frac{|\theta\cdot x|}{\|\theta\|_2}\geq \gamma$. Then, for every distribution $D$, we have that
\begin{align*}
    \|\theta\|_2 \gamma 
    & \leq \Ex_{(x,y)\sim D}[|\theta\cdot x|]\\
    & = \Ex_{(x,y)\sim D}[f(x)\cdot \theta\cdot x]\\
    & =\sum_{j\in V} (\theta_j \Ex_{(x,y)\sim D}[f(x)\cdot x_j]) \\
    & \leq \|\theta\|_1 \max_{j\in V} |\Ex_{(x,y)\sim D}[f(x)\cdot x_j]|
\end{align*}
The proof is complete by observing that $\frac{\|\theta\|_2}{\|\theta\|_1}\geq \frac{1}{\sqrt{|V|}}\geq \frac{1}{\sqrt{k}}$.
\end{proof}

\begin{figure}[H]
  \captionsetup{width=.9\linewidth}
\begin{tcolorbox}[colback = white,arc=1mm, boxrule=0.25mm]
\vspace{3pt} 

{\sc Boost}$(S^{(1)}, \ldots, S^{(n)}, \mcH, t)$ :  \vspace{6pt} \\
\textbf{Input:} Samples $S^{(1)}, \ldots, S^{(n)}$ each with $m$ points, a concept class $\mcH$, and a step count $t$.\\
\textbf{Output:} A hypothesis for each of the $n$ tasks. \\

\ \ Initialize the hypotheses $h^{(1)}, \ldots, h^{(n)}$ each to the constant $0$ functions. \vspace{4pt} 

\ \ Repeat $t$ times: 
\vspace{-3pt}

\begin{enumerate}
\item (Reweight points). For each $i \in [n]$ and $j \in [m]$, set
\begin{equation*}
    w^{(i)}_j = \exp(-y^{(i)}_j \cdot h^{(i)}(x^{(i)}_j))
\end{equation*}
and for each $i \in [n]$ set
$
    W^{(i)} = \sum_{j \in [m]} w^{(i)}_j.
$
\item (Choose a weak learner). Choose $h^\star$ to maximize
\begin{equation}
    \label{eq:pseudocode-advantage}
    h^\star = \argmax_{h \in \mathcal{H}} \sum_{i \in [n]} W^{(i)} \cdot \left(\sum_{j \in [m]} \frac{w^{(i)}_j}{W^{(i)}} \cdot y_j^{(i)} h(x_j^{(i)})\right)^2
\end{equation}

\item (Update hypotheses). For each $i \in [n]$, update $h^{(i)} \leftarrow h^{(i)} + \alpha^{(i)} \cdot h^\star$ where
\begin{equation*}
    \alpha^{(i)} = \frac{1}{2} \ln \left(\frac{\sum_{j \in [m]} w_j^{(i)} \cdot \Ind[ y_j^{(i)} = h^\star(x_j^{(i)})]}{\sum_{j \in [m]} w_j^{(i)} \cdot \Ind[ y_j^{(i)} \neq h^\star(x_j^{(i)})]}\right).
\end{equation*}
\end{enumerate}
\ \ Output $h^{(1)}, \ldots, h^{(n)}$.
\end{tcolorbox}
\caption{Pseudocode for simultaneous boosting}
\label{fig:pseudocode}
\end{figure}

Our simultaneous boosting algorithm will use a generalization of the weak-learning hypothesis.
\begin{definition}[Simultaneous weak-learning assumption]
    A class of weak learners, $\mcH$, satisfies the $\Gamma$-simultaneous weak-learning assumption for functions $f^{(1)}, \ldots, f^{(n)}$, if for all input distributions $D^{(1)}, \ldots, D^{(n)}$ and nonnegative weights $w_1, \ldots, w_n$, there exists $h \in \mcH$ satisfying
    \begin{equation*}
        \sum_{i \in [n]} w_i \cdot \Ex_{(x,y) \sim D^{(i)}}[f^{(i)}(x)h(x)]^2 \geq \Gamma \cdot \sum_{i \in [n]}w_i.
    \end{equation*}
\end{definition}

\begin{lemma}[Simultaneous discriminator lemma]
    \label{lem:multi-disc}
    Let $f^{(1)}, \ldots, f^{(n)}: \bits^d \to \bits$ be halfspaces with margin $\gamma$ over a set of features $V\subseteq [d]$, such that $|V|\leq k$. Then, the class of single feature projection functions, $\mcH_\proj \coloneqq \{x \mapsto x_\ell \mid \ell \in [d]\}$ satisfies the $(\Gamma = \frac{\gamma^2}{k^2})$-simultaneous weak-learning assumption for functions $f^{(1)}, \ldots, f^{(n)}$
\end{lemma}

\begin{proof}
    Let $\mcH_{\mathrm{relevant}} \subseteq \mcH_\proj$ be the projection functions corresponding to the $|V|\leq k$ relevant features. Then,
    \begin{align*}
        \max_{h \in \mcH_{\mathrm{relevant}}} \left(\sum_{i \in [n]} w_i \cdot \Ex_{x \sim D^{(i)}}[f^{(i)}(x)h(x)]^2 \right)
        &\geq \frac{1}{|\mcH_{\mathrm{relevant}}|} \cdot \sum_{h \in \mcH_{\mathrm{relevant}}}  \sum_{i \in [n]} w_i \cdot \Ex_{x \sim D^{(i)}}[f^{(i)}(x)h(x)]^2 \\
        & \geq \frac{1}{k} \cdot \sum_{i \in [n]} w_i \cdot \max_{h \in \mcH_{\mathrm{relevant}}}  \left(\Ex_{x \sim D^{(i)}}[f^{(i)}(x)h(x)]^2\right)\\
        &\geq \frac{1}{k} \cdot \sum_{i \in [n]} w_i \cdot \frac{\gamma^2}{k} \tag{\Cref{fact:disc-lemma}} \\
        & \geq \frac{\gamma^2}{k^2}  \tag{$\sum_{i \in [n]} w_i = 1$}
    \end{align*}
\end{proof}


\begin{lemma}[Simultaneous boosting fits a training set]
    \label{lem:boost-train-acc}
    For any samples $S^{(1)},\ldots, S^{(n)}$ of size $m$, weak-learning class $\mcH$, and number of steps $t$, let $h^{(1)}, \ldots, h^{(n)}$ be the output returned by {\sc Boost}$(S^{(1)}, \ldots, S^{(n)}, \mcH, t)$. For each $s \in [t]$, let $\Gamma_s \in [0,1]$ be unique value that satisfies the following expression when $h^{\star}$ is chosen in the $s^{\text{th}}$ iteration:
    \begin{equation*}
        \sum_{i \in [n]} W^{(i)} \cdot \left(\sum_{j \in [m]} \frac{w^{(i)}_j}{W^{(i)}} \cdot y_j^{(i)} h^\star(x_j^{(i)})\right)^2 = \Gamma_s \cdot \sum_{i \in [n]} W^{(i)}
    \end{equation*}
    Then,
    \begin{equation*}
        \frac{1}{nm} \sum_{i \in [n]} \sum_{j \in [m]} \Ind[\sign(h^{(i)}(x^{(i)}_j)) \neq y^{(i)}_j] \leq \prod_{s \in [t]}\left(1 - \frac{\Gamma_s}{2}\right)
    \end{equation*}
\end{lemma}
Note that $\Gamma_s$ is defined in such a way that, if $\mcH$ satisfies the $\Gamma$-simultaneous weak-learning assumption, then $\Gamma_s \geq \Gamma$ for all iterations.

\begin{proof}[Proof of \Cref{lem:boost-train-acc}]
    For each $s \in [t]$, let $h^{(1)}_s, \ldots, h^{(n)}_s$ be the hypothesis at the end of the $s^{(th)}$ iteration, with $s = 0$ use to denote the start of the algorithm. We will track the exponential loss,
    \begin{equation*}
        L_s \coloneqq \sum_{i \in [n]} \sum_{j \in [m]} \exp\left(-y_j^{(i)} \cdot h^{(i)}_s(x_j^{(i)})\right).
    \end{equation*}
    We will prove, by induction, that $L_s \leq nm \cdot \prod_{s \in [t]}\left(1 - \frac{\Gamma_s}{2}\right)$. The base case of $s = 0$ holds with equality. For any $s \geq 1$, let $w_j^{(i)}$ and $W^{(i)}$ be the weights during the $s^{\text{th}}$ iteration. Then,
    \begin{align*}
        L_s &=  \sum_{i \in [n]} \sum_{j \in [m]} \exp\left(-y_j^{(i)} \cdot h^{(i)}_s(x_j^{(i)})\right)\\
        &=  \sum_{i \in [n]} \sum_{j \in [m]} \exp\left(-y_j^{(i)} \cdot\left(h^{(i)}_{s-1}(x_j^{(i)}) + \alpha^{(i)}_s \cdot h^\star_s(x_j^{(i)})\right)\right)\\
        &=  \sum_{i \in [n]} \sum_{j \in [m]} w_j^{(i)} \cdot \exp\left(-y_j^{(i)} \cdot \alpha^{(i)}_s \cdot h^\star_s(x_j^{(i)})\right).
    \end{align*}
    For each $i \in [n]$, we'll use the shorthand:
    \begin{align*}
        W^{(i)}_{=} &\coloneqq \sum_{j \in [m]} w_j^{(i)} \cdot \Ind[ y_j^{(i)} = h^\star(x_j^{(i)})], \\
        W^{(i)}_{\neq} &\coloneqq \sum_{j \in [m]} w_j^{(i)} \cdot \Ind[ y_j^{(i)} \neq h^\star(x_j^{(i)})].
    \end{align*}
    Hence,
    \begin{align*}
        L_s &=  \sum_{i \in [n]}  W^{(i)}_{=} \cdot \exp(-\alpha^{(i)}_s) + W^{(i)}_{\neq} \cdot \exp(\alpha^{(i)}_s) \\
        &=  \sum_{i \in [n]}  W^{(i)}_{=} \cdot \sqrt{\frac{W^{(i)}_{\neq}}{W^{(i)}_{=}}} + W^{(i)}_{\neq} \cdot \sqrt{\frac{W^{(i)}_{=}}{W^{(i)}_{\neq}}} \\
        &=  \sum_{i \in [n]} 2 \sqrt{W^{(i)}_{=}\cdot W^{(i)}_{\neq}} 
    \end{align*}
    Note that $W^{(i)}_{=} +  W^{(i)}_{\neq} = W^{(i)}$. For each $i \in [n]$, we define
    \begin{equation*}
        \gamma^{(i)} \coloneqq \sum_{j \in [m]} \frac{w^{(i)}_j}{W^{(i)}} \cdot y_j^{(i)} h^\star(x_j^{(i)}).
    \end{equation*}
    and observe $W_{=}^{(i)} - W_{\neq}^{(i)} = W^{(i)} \cdot \gamma^{(i)}$. As a result, we have that $W^{(i)}_{=} = W^{(i)}/2 \cdot (1 + \gamma^{(i)})$ and $W^{(i)}_{\neq} = W^{(i)}/2 \cdot (1 - \gamma^{(i)})$. Continuing,
    \begin{align*}
        L_s &=  \sum_{i \in [n]} 2 \sqrt{W^{(i)}_{=}\cdot W^{(i)}_{\neq}}  \\
        &= \sum_{i \in [n]} W^{(i)} \cdot \sqrt{(1 + \gamma^{(i)})\cdot(1 - \gamma^{(i)})} \\
        &\leq  \sum_{i \in [n]} W^{(i)} \cdot (1 - \frac{(\gamma^{(i)})^2}{2} ) \\
        &= L_{s-1} -  \sum_{i \in [n]} W^{(i)} \cdot  \frac{(\gamma^{(i)})^2}{2}  \\
        &= L_{s-1} - \frac{1}{2} \cdot \sum_{i \in [n]} W^{(i)} \left(\sum_{j \in [m]} \frac{w^{(i)}_j}{W^{(i)}} \cdot y_j^{(i)} h^\star(x_j^{(i)}) \right)^2 \\
        &= L_{s-1} - \frac{\Gamma_s}{2} \cdot \sum_{i \in [n]} W^{(i)} \\
        &= L_{s-1}\cdot \left(1 - \frac{\Gamma_s}{2}\right) \tag{$\sum_{i \in [n]} W^{(i)} = L_{s-1}$}
    \end{align*}
    
    Hence, we have that $L_t \leq nm \cdot \prod_{s \in [t]}\left(1 - \frac{\Gamma_s}{2}\right)$. The desired holds because classification error is upper bounded by $L_t/(nm)$.
\end{proof}

As an immediate corollary, we have that if $\mcH$ satisfies the simultaneous weak-learning assumption, for an appropriate choice of $t$, simultaneous boosting will fit the training set with almost perfect accuracy.

\begin{corollary}
    \label{cor:fit-training-set}
    Let any functions $f^{(1)},\ldots,f^{(n)}$ and $\mcH$ be a class of weak-learners satisfying the $\Gamma$-simultaneous weak-learning assumption for $f^{(1)}, \ldots, f^{(n)}$. Then, for any samples $S^{(1)}, \ldots, S^{(n)}$ of size $m$ labeled by $f^{(1)}, \ldots, f^{(n)}$ and $t = O\left(\frac{\log(1/\eps)}{\Gamma}\right)$, {\sc Boost}$(S^{(1)}, \ldots, S^{(n)}, \mcH, t)$ returns hypotheses $h^{(1)},\ldots,h^{(n)}$ such that $\widehat{\avgerror}(h^{(1)},\ldots,h^{(n)})\leq \eps$.
    \footnote{For conciseness, here we slightly abuse notation by writing $\widehat{\avgerror}(h^{(1)},\ldots,h^{(n)})$ to denote the average training error of the classification functions $\mathrm{sign}(h^{(i)}(x))$.}
\end{corollary}

We also bound the running time of simultaneous boosting.
\begin{proposition}[Running time]
    \label{prop:time}
    The running time of {\sc Boost}$(S^{(1)}, \ldots, S^{(n)}, \mcH, t)$ when each $S^{(i)}$ has $m$ samples is $O(nmt |\mcH|)$.
\end{proposition}
\begin{proof}
    The running time is dominated by finding which weak-learner maximizes \Cref{eq:pseudocode-advantage}. To do so so, we can loop over all $|\mcH|$ weak-learners and compute their advantage, which takes time $O(nm)$. This must be done in each of $t$ iterations, given a total runtime of $O(nmt |\mcH|)$.
\end{proof}

The last step in proving \Cref{th:upperbound} is bounding the generalization error. We first prove a general theorem on bounding generalization in the multitask setting using VC dimension in~\Cref{subsec:gen-VC} and then apply it to learning large-margin halfspaces to complete the proof of the theorem in~\Cref{subsec:gen-proof}. 


\subsection{Generalization based on VC dimension}
\label{subsec:gen-VC}

The goal of this section is to prove the following theorem bounding the number of samples needed to generalize in the multitask setting. The formal version is given in \Cref{thm:generalization}.
\begin{theorem}[Generalization in the multitask setting]
    \label{thm:informal-generalization}
    Let $\mcC$ be a class of functions $f: \bits^d \to \bits$ and $\mcV \subseteq 2^{[d]}$ be subsets of features. For $\VC(\mcC \mid \mcV)$ as defined in~\Cref{def:VC-condition}, any $\eps, \delta \in (0,1)$, and
    \begin{align*}
        m &= O\left(\VC(\mcC \mid \mcV) \cdot \frac{\log(1/\eps)}{\eps}\right), \\
        nm &= O\left(\frac{\log |\mcV| + \log(1/\delta)}{\eps}\right),
    \end{align*}
    given random size-$m$ samples for each of $n$ tasks, any $h^{(1)}, \ldots, h^{(n)}\in \mcC$ with a shared-feature representation $V\in\mcV$ 
    with $\widehat{\avgerror}(h^{(1)}, \ldots, h^{(n)})\leq \eps$, 
    will have $\avgerror(h^{(1)}, \ldots, h^{(n)}) \leq 4\eps$ with probability at least $1 - \delta$.
\end{theorem}
$\VC(\mcC \mid \mcV)$ will correspond to the VC dimension of the concept class once a representation is fixed. It can be substantially smaller than $\VC(\mcC)$. For example, if $\mcC$ is the set of all $k$-sparse halfspaces for $k \ll d$, then $\VC(\mcC) = \Theta(k \log d)$. 
However, if all $n$ tasks correspond to a halfspace over the same $k$ features, then we take $\mcV \coloneqq \{V \subseteq 2^{[d]} \mid |V| \leq k\}$, and have $\VC(\mcC \mid \mcV) = k + 1$. Once a representation is fixed, $\mcC$ just corresponds to halfspaces over a set of $k$ features, which has VC dimension $\Theta(k)$.

\Cref{thm:informal-generalization} roughly speaking, says that each task need only have enough samples to learn assuming the representation $V \in \mcV$ is already known, and the total number of samples should be enough to learn which representation $V$ is used. To formalize \Cref{thm:informal-generalization}, we begin with some basic definitions.
\begin{definition}[Generalization failure probability, single-task setting]
    For a concept class $\mcC$ of functions $f: X \to \bits$, distribution $D$ over $X \times \bits$, error parameter $\eps > 0$, and sample size $m$, we define $\delta_\gen(\mcC, m, \eps, D)$ to be the probability, over a random sample $S$ of $m$ points from $D$, 
    that there exists some $f \in \mcC$ that has at most $\eps$ error on the sample $S$ but for which $\Pr_{(x, y) \sim D}[f(x) \neq y] \geq 4\eps$. We define the generalization failure probability of $\mcC$ with sample size $m$ and error parameter $\eps$ to be
    \begin{equation*}
        \delta_\gen(\mcC, m, \eps) \coloneqq \sup_{\text{distribution } D} \delta_\gen(\mcC, m, \eps, D).
    \end{equation*}
\end{definition}
Any algorithm that returns a hypothesis $h$ within $\mcC$ with less than $\eps$ error on $m$ random samples will learn to $\error < 4\eps$ with probability at least $1 - \delta_\gen(\mcC, m, \eps)$. We extend this notion to the multitask setting.
\begin{definition}[Generalization failure probability, multitask setting]
    For a concept class $\mcC$ of functions $f: \bits^d \to \bits$, $\mcV \subseteq 2^{[d]}$ a collection of subsets of features, distributions $D^{(1)}, \ldots, D^{(n)}$ over $\bits^d \times \bits$, error parameter $\eps$, number of tasks $n$, and samples-per-task $m$, we define $\delta_\gen(\mcC, \mcV, n, m, \eps, D^{(1)}, \ldots, D^{(n)})$ to be the probability over random samples $S^{(i)} \sim (D^{(i)})^m$, $i\in[n]$, that there are $h^{(1)}, \ldots, h^{(n)} \in \mcC$ satisfying $\bigcup_{i\in[n]} \rel(h^{(i)})\in \mcV$ for which 
    $\widehat{\avgerror}(h^{(1)},\ldots,h^{(n)})\leq \eps$
    but for which $\avgerror(h^{(1)}, \ldots, h^{(n)}) \geq 4\eps$, i.e.,
    \[\frac{1}{n}\sum_{i\in [n]}\Prx_{(x, y)\sim D^{(i)}}[h^{(i)}(x)\neq y]\geq 4\eps.\]
    Then, we define,
    \begin{equation*}
        \delta_\gen(\mcC, \mcV, n, m, \eps) \coloneqq \sup_{\text{distributions } D^{(1)}, \ldots, D^{(n)}} \delta_\gen(\mcC, \mcV, n, m, \eps, D^{(1)}, \ldots, D^{(n)}).
    \end{equation*}
\end{definition}

Once again, any algorithm that returns hypotheses $h^{(1)}, \ldots, h^{(n)}$ with average error at most $\eps$ on the training set, will, given $m$ samples per task, with probability $1 - \delta_\gen(\mcC, \mcV, n,m,\eps)$, learn hypotheses with $\avgerror(h^{(1)}, \ldots, h^{(n)}) < 4\eps$. We extend classical results bounding $\delta_\gen(\mcC, m, \eps)$ based on VC dimension to the multitask setting. We begin with the basic definitions of VC theory.

\begin{definition}[VC dimension~\cite{VapnikC71}]
    For any concept class $\mathcal{C}$ and domain $X$, the shattering number of $n$ points is defined as 
    \begin{equation*}
        \Pi_{\mathcal{C}}(n) \coloneqq\max_{x_1, \ldots, x_n \in X} \left|\left\{(c(x_1), \ldots, c(x_n))  : c \in \mcC \right\}\right|
    \end{equation*}
    and is the maximum number of unique assignments functions in the concept class can take on a set of $n$ points. The VC dimension of $\mcC$, denoted     \begin{equation*}
        \VC(\mcC) \coloneqq \sup \{d\in\mathbb{N} :  \Pi_{\mcC}(d) = 2^d\},
    \end{equation*}
    is the cardinality of the largest data set for which every unique assignment to that set is satisfied by a function in $\mcC$.
\end{definition}

In the multitask setting, rather than scaling with the VC dimension of $\mcC$, the number of samples needed per task will scale with the VC dimension after the representation is already known.
\begin{definition}[VC dimension given representation]
    \label{def:VC-condition}
    Given a concept class $\mcC$ of functions $f: \bits^d \to \bits$, and $\mcV \subseteq 2^{[d]}$ a collection of subsets of features, for any $V \subseteq [d]$, let
    \begin{equation*}
        (\mcC \mid V) \coloneqq \{f \in \mcC : \rel(f) \subseteq V\}, 
    \end{equation*}
    be the concepts consistent with a shared representation $V$. We define the VC dimension of $\mcC$ given the representation $\mcV$ to be
    \begin{equation*}
        \VC(\mcC \mid \mcV) \coloneqq \max_{V \in \mcV} \VC(\mcC \mid V).
    \end{equation*}
\end{definition} 
We are most interested in settings where $\VC(\mcC \mid \mcV) \ll \VC(\mcC)$. We can now formalize the main result of this section.

\begin{theorem}[Generalization bound in the multitask setting, formal version of \Cref{thm:informal-generalization}]
    \label{thm:generalization} Given a concept class $\mcC$ of functions $f: \bits^d \to \bits$, $\mcV \subseteq 2^{[d]}$ a collection of subsets of features, and accuracy parameters $\eps, \delta$, set $m,n$
    \begin{align*}
        m &= O\left(\VC(\mcC \mid \mcV)\cdot \frac{\log(1/\eps)}{\eps}\right), \\
        nm &= O\left(\frac{\log |\mcV| + \log(1/\delta)}{\eps}\right).
    \end{align*}
    Then $\delta_\gen(\mcC, \mcV, n, m, \eps) \leq \delta$.
\end{theorem}
The proof of \Cref{thm:generalization} follows Blumer, Ehrenfeucht, Haussler, and Warmuth's classical VC generalization bounds \cite{BEHW89}, with appropriate modifications made for the multitask setting. To do so, we need to define an extension of shattering numbers to the multitask setting.
\begin{definition}
    For any function $f:X \to Y$ and sample $S = (x_1, \ldots, x_m) \in X^m$, we use $f(S)$ as shorthand for the vector $(f(x_1), \ldots, f(x_m))$. For any concept class $\mcC$ of functions $f:\bits^d \to \bits$ and $\mcV \subseteq 2^{[d]}$ a collection of subsets of features, we define the shattering number for $n$ tasks and $m$ samples-per-task, denoted $\Pi_{\mathcal{C}, \mathcal{V}}(n, m)$, to be
    \begin{equation*}
        \max_{S^{(1)}, \ldots, S^{(n)} \in (\bits^d)^m} \left|\left\{(f^{(1)}(S^{(1)}), \ldots, f^{(n)}(S^{(n)}))  : f^{(1)}, \ldots, f^{(n)} \in \mcC \text{ and } \cup_{i \in [n]} \rel(f^{(i)}) \in \mcV \right\}\right|
    \end{equation*}
    to be the maximum number of unique assignments for $n$ tasks each with $m$ data points.
\end{definition}

We can now state the main technical lemma of this section.
\begin{lemma}
    \label{lem:shattering-to-generalization-multi}
    For any concept class $\mcC$ of functions $f:\bits^d \to \bits$,  $\mcV \subseteq 2^{[d]}$ a collection of subsets of features, error parameter $\eps > 0$, number of tasks $n$, and samples-per-task $m$, if $nm\eps \geq 2$, then
    \begin{equation*}
        \delta_\gen(\mcC, \mcV, n, m, \eps) \leq 2\cdot \Pi_{\mcC,\mcV}(n,2m) \cdot \exp\left(-nm\eps/10\right).
    \end{equation*}
\end{lemma}

We'll collect two basic probability facts that will be used in the proof of \Cref{lem:shattering-to-generalization-multi}.
\begin{fact}[Application of Chebyshev's inequality]
    \label{fact:rv-large}
    Let $x$ be a random variable in $\R$ with mean $\mu$ and for which $\Var[x] \leq \frac{\mu^2}{8}$. Then,
    \begin{align*}
        \Pr\left[x \geq \mu/2 \right] \geq \frac{1}{2}.
    \end{align*}
\end{fact}

The second probability fact we need is a slight twist on standard Chernoff bounds.
\begin{proposition}
    \label{prop:sampling-without-replacement}
    For any $m, n, \ell \in \N$, suppose there are $n$ groups of $2m$ items, and of the $2nm$ items, $\ell$ are marked. If we uniformly select $m$ items from each of the $n$ groups \emph{without replacement}, the probability at least $\frac{2\ell}{3}$ items are selected in total is at most $\exp(-\ell/20))$.
\end{proposition}
\begin{proof}
    The expected number of items selected is $\mu = \ell/2$. First, suppose that we selected the items \emph{with replacement}. Then, whether each marked item is selected is independent. In this setting, by a standard Chernoff bound, we have that the probability at least $\frac{2\ell}{3} = (1 + \frac{1}{2}) \mu$ items are selected is at most $\exp(-\mu/10) = \exp(-\ell/20)$.
    
    At a high level, sampling \emph{without replacement} can only improve this bound. In more detail, let $\bX$ be the number of marked items selected when sampling with replacement, and $\bY$ be for the setting of sampling without replacement. The only information about $\bX$ needed for the above Chernoff bound to hold is an upper bound on the moment generating function $\Ex[e^{\lambda \bX}]$ for appropriately chosen $\lambda$.  It is therefore sufficient to argue that $\Ex[e^{\lambda \bY}] \leq \Ex[e^{\lambda \bX}]$ for every $\lambda \in \R$.
    
    Let $\bx_1, \ldots, \bx_n$ and $\by_1, \ldots, \by_n$ be random variables indicating the number of marked items selected from each of the $n$ groups when sampling with replacement and without replacement respectively. As $\bx_1, \ldots, \bx_n$ are independent,
    \begin{equation*}
        \Ex[e^{\lambda \bX}] = \prod_{i \in [n]} \Ex[e^{\lambda \bx_i}].
    \end{equation*}
    Similarly $\by_1, \ldots, \by_n$ are independent, as the items selected in one group do not effect the items selected in other groups, so,
    \begin{equation*}
        \Ex[e^{\lambda \bY}] = \prod_{i \in [n]} \Ex[e^{\lambda \by_i}].
    \end{equation*}
    It is therefore sufficient to prove that $\Ex[e^{\lambda \by_i}] \leq \Ex[e^{\lambda \bx_i}]$ for each $i \in [n]$. This is proven in \cite[Theorem 4]{hoeffding}, which states that sampling without replacement can only decrease the moment generating function.
\end{proof}


We proceed to prove \Cref{lem:shattering-to-generalization-multi}. This proof mostly follows the exposition in \cite[\S 3.5.2]{KV94} of \cite{BEHW89}'s classic result, with appropriate modifications for the multitask setting.
\begin{proof}[Proof of \Cref{lem:shattering-to-generalization-multi}]
    Fix distributions $D^{(1)}, \ldots, D^{(n)}$ each over $\bits^d \times \bits$. For each $i \in [n]$, let $S^{(i)} \sim (D^{(i)})^m$ be a size-$m$ random sample, and let $\bA$ be the event that there are $h^{(1)}, \ldots, h^{(n)} \in \mcC$ satisfying $\bigcup_{i\in[n]} \rel(h^{(i)})\in \mcV$ for which 
    \begin{enumerate}
        \item At most $\eps$-fraction of points in $S$ are misclassified:
        \begin{equation*}
            \sum_{(x,y) \in S^{(i)}}\Ind[h^{(i)}(x)\neq y] \leq nm\eps.
        \end{equation*}
        \item The population error satisfies $\avgerror(h^{(1)}, \ldots, h^{(n)}) \geq 4\eps$, i.e.,
    \[\frac{1}{n}\sum_{i\in [n]}\eps^{(i)}\geq 4\eps \quad\quad\text{where } \quad \eps^{(i)} \coloneqq \Prx_{(x, y)\sim D^{(i)}}[h^{(i)}(x)\neq y].\]
    \end{enumerate}
    
    Our goal is to upper bound $\Pr[\bA]$. Suppose we draw fresh samples $T^{(i)} \sim (D^{(i)})^m$ for each $i \in [n]$. Let $z^{(i)}_j$ indicate whether $h^{(i)}$ misclassifies the $j^{\text{th}}$ point in $T^{(i)}$. Then, $z^{(i)}_j \sim \mathrm{Ber}(\eps^{(i)})$, and the $z^{(i)}_j$ are independent across $i \in [n], j \in [m]$. Using $Z = \sum_{i \in [n], j \in [m]} z^{(i)}_j$ to indicate the total number of points misclassified on the fresh samples, we have
    \begin{align*}
         \Ex[Z] &= m \sum_{i \in [n]} \eps^{(i)} \geq 4nm \eps. \\
         \Var[Z] &= m \sum_{i \in [n]} \eps^{(i)}(1 - \eps^{(i)}) \leq \Ex[Z].
    \end{align*}
    Applying \Cref{fact:rv-large}, as long as $4nm\eps \geq 8$, $\Pr[Z \geq 2nm\eps] \geq \lfrac{1}{2}$. 
    Let $\bB$ be the event, depending on both the original $S$ samples and the fresh $T$ samples, that there are $h^{(1)}, \ldots, h^{(n)} \in \mcC$ satisfying $\bigcup_{i\in[n]} \rel(h^{(i)})\in \mcV$ meeting the following three criteria.
    \begin{enumerate}
        \item At most $\eps$-fraction of points in $S$ are misclassified:
        \begin{equation*}
           \sum_{i\in [n]}\sum_{(x,y) \in S^{(i)}}\Ind[h^{(i)}(x)\neq y] \leq nm\eps.
        \end{equation*}
        \item The average test error is at least 4$\eps$:
        \[\frac{1}{n}\sum_{i\in [n]}\Prx_{(x, y)\sim D^{(i)}}[h^{(i)}(x)\neq y]\geq \eps.\]
        \item At least $2\eps$-fraction of points in $T$ are misclassified:
        \begin{equation*}
            \sum_{i\in [n]}\sum_{(x,y) \in T^{(i)}}\Ind[h^{(i)}(x)\neq y] \geq 2nm\eps.
        \end{equation*}
    \end{enumerate}
    Due to the first two criteria, $\bB$ can only occur if $\bA$ occurs. Furthermore, $\Pr[\bB \mid \bA] = \Pr[Z \geq 2nm\eps] \geq \lfrac{1}{2}$. As a result, we have that
    \begin{equation*}
        2\Pr[\bB] = 2(\Pr[\bB \mid \bA] \Pr[\bA]) \geq 2(\frac{1}{2} \Pr[\bA]) = \Pr[\bA].
    \end{equation*}
    Therefore, in order to upper bound $\Pr[\bA]$, it is sufficient to upper bound $\Pr[\bB]$. Indeed we will show that just the first and third criteria of $\bB$ are unlikely to occur together.
    
    We consider an alternative and equivalent generation process for the samples. For each task $i \in [n]$, we draw $2m$ samples from $D^{(i)}$. Then, we partition half of those samples into $S^{(i)}$ and the other half into $T^{(i)}$. 
    
    Consider a single possible labeling for all $2nm$ points. Let $\ell$ be the total number of misclassified points, and $\ell_S$ and $\ell_T$ be the number of misclassified points that are partitioned into $S$ and $T$ respectively. Then, in order for $\bB$ to occur, it must be the case that $\ell_S \leq nm\eps$ and $\ell_T \geq 2nm\eps$. In particular, this implies that $\ell \geq 2nm\eps$ and $\frac{\ell_T}{\ell} \geq \frac{2}{3}$.

    
    
     By \Cref{prop:sampling-without-replacement}, for a single labeling of the $2nm$ points, $\bB$ occurs with probability at most $\exp(-\frac{nm\eps}{10})$. There are only $\Pi_{\mcC,\mcV}(n,2m)$ possible labelings for the $2nm$ points. Therefore, we can upper bound,
    \begin{align*}
        \Pr[\bA] \leq 2\Pr[\bB] \leq 2\cdot \Pi_{\mcC,\mcV}(n,2m) \cdot \exp\left(-nm\eps/ 10\right).
    \end{align*}


\end{proof}

We are now nearly ready to prove \Cref{thm:generalization}. As in classical VC theory, we'll apply the Sauer-Shelah Lemma.
\begin{fact}[Sauer-Shelah Lemma~\cite{Sauer72,Shelah72}]
    \label{fact:sauer-lemma}
    For any concept class $\mcC$ and samples size $m \geq \VC(\mcC)$,
    \begin{equation*}
        \Pi_C(m) \leq \left(\frac{em}{\VC(\mcC)}\right)^{\VC(\mcC)}.
    \end{equation*}
\end{fact}
We'll plug in \Cref{fact:sauer-lemma} as a blackbox into the following proposition to bound shattering numbers in the multitask setting.
\begin{proposition}
    \label{prop:bound-shattering-multi}
    For any concept class $\mcC$ of functions $f: \bits^d \to \bits$, $\mcV \subseteq 2^{[d]}$ a collection of subsets of features and sample size $m \geq \VC(\mcC \mid \mcV)$,
    \begin{align}
        \Pi_{\mcC, \mcV}(n,m) &\leq \sum_{V \in \mcV} (\Pi_{(\mcC \mid V)}(m))^n  \label{eq:bound-shattering}\\
        &\leq |\mcV| \cdot \left(\frac{em}{\VC(\mcC \mid \mcV)}\right)^{n \cdot \VC(\mcC \mid \mcV)}. \label{eq:bound-shattering-sauer}
    \end{align}
\end{proposition}
\begin{proof}
    First, we prove \Cref{eq:bound-shattering}. For any $V \in \mcV$, given that $\bigcup_{i\in[n]} \rel(f^{(i)}) = V$, the number of unique ways that $f^{(i)}$ can classify a sample of size $m$ is at most $\Pi_{(\mcC \mid V)}(m)$. Therefore, the total ways to classify all $nm$ points is at most $(\Pi_{(\mcC \mid V)}(m))^n$. Summing over all $V \in \mcV$ gives \Cref{eq:bound-shattering}.
    
    To prove \Cref{eq:bound-shattering-sauer}, we just plug in \Cref{fact:sauer-lemma}, giving $\Pi_{(\mcC \mid V)}(m) \leq \left(\frac{em}{\VC(\mcC \mid \mcV)}\right)^{\VC(\mcC \mid \mcV)}$
\end{proof}

Finally, we have all the pieces in place to prove \Cref{thm:generalization}
\begin{proof}[Proof of \Cref{thm:generalization}]
    For the $m,n$ given, we'll have $m \geq \VC(\mcC \mid \mcV)$, so \Cref{prop:bound-shattering-multi} applies, and $nm\eps \geq 2$, so \Cref{lem:shattering-to-generalization-multi} applies. Therefore,
    \begin{align*}
        \delta_\gen(\mcC, \mcV, n, m, \eps) &\leq 2\cdot \Pi_{\mcC,\mcV}(n,2m) \cdot \exp\left(-nm\eps/ 10\right) \tag{\Cref{lem:shattering-to-generalization-multi}} \\
        &\leq 2|\mcV| \cdot \left(\frac{2em}{\VC(\mcC \mid \mcV)}\right)^{n \cdot \VC(\mcC \mid \mcV)} \cdot \exp\left(-nm\eps/ 10\right) \tag{\Cref{prop:bound-shattering-multi}}
    \end{align*}
    By setting $m = O\left(\VC(\mcC \mid \mcV)\cdot \frac{\log(1/\eps)}{\eps}\right)$, we have that
    \begin{equation*}
         \left(\frac{2em}{\VC(\mcC \mid \mcV)}\right)^{n \cdot \VC(\mcC \mid \mcV)} \leq\exp\left(nm\eps/ 20\right).
    \end{equation*}
    As a result, we can bound
    \begin{equation*}
        \delta_\gen(\mcC, \mcV, n, m, \eps) \leq 2|\mcV| \cdot \exp\left(-nm\eps/ 20\right)
    \end{equation*}
    which is at most $\delta$ when $nm = O\left(\frac{\log |\mcV| + \log(1/\delta)}{\eps}\right)$, as desired.
\end{proof}

\subsection{Proof of~\Cref{th:upperbound}}\label{subsec:gen-proof}

Here, we apply the machinery from \Cref{subsec:gen-VC} to complete the proof of \Cref{th:upperbound}. 

\begin{proof}[Proof of \Cref{th:upperbound}]
    Pick any $f^{(1)}, \ldots, f^{(n)} \in \mcC_{\gamma}$ with shared representation in $\mcV_k$. By \Cref{lem:multi-disc}, the class of projection functions $\mcH_\proj \coloneqq \{x \mapsto x_\ell \mid \ell \in [d]\}$ satisfies the $(\Gamma = \frac{\gamma^2}{k^2})$-simultaneous weak-learning assumption for $f^{(1)},\ldots, f^{(n)}$. 
    
    Let $S^{(1)}, \ldots, S^{(n)}$ be samples of $m$ points for each of the $n$ tasks. By \Cref{cor:fit-training-set}, for,
    \begin{equation*}
        t = O\left(\frac{\log(4/\eps)}{\Gamma}\right) =  O\left(\frac{k^2\log(1/\eps)}{\gamma^2}\right)
    \end{equation*}
    the output of {\sc Boost}$(S^{(1)}, \ldots, S^{(n)}, \mcH_\proj, t)$ will have average training error at most $\eps/4$.
    
    Next, we prove generalization. The hypotheses output by {\sc Boost} all depend on the same $t$ projection functions, so have a shared representation in $\mcV_t=\{V\subseteq [d] \mid |V|\leq t\}$ which satisfies $\log |\mcV_t| = O(t \log d)$. Once a representation $V \subseteq [d]$ is fixed the class $(\mcC_{\gamma} \mid V)$ consists of $\gamma$-margin halfspaces of the coordinates in $V$, which has VC dimension at most $|V|$. Therefore, $\VC(\mcC_\gamma \mid \mcV) \leq t$. Applying \Cref{thm:generalization} gives that, if $n,m$ are chosen so that 
    \begin{align*}
        m &= O\left(t\cdot \frac{\log(1/\eps)}{\eps}\right)=O\left(\frac{k^2\log^2(1/\eps)}{\gamma^2\eps}\right), \\
        nm &= O\left(\frac{t\log(d) + \log(1/\delta)}{\eps}\right)=O\left(\frac{(k^2/\gamma^2)\log(1/\eps)\log(d) + \log(1/\delta)}{\eps}\right),
    \end{align*}
    then with probability $1 - \delta$, hypotheses $h^{(1)},\ldots, h^{(n)}$ output by {\sc Boost}$(S^{(1)}, \ldots, S^{(n)}, \mcH_\proj, t)$ have $\avgerror(h^{(1)},\ldots, h^{(n)}) \leq \eps$.
    
    Finally, the runtime is bounded by \Cref{prop:time}.
\end{proof}

\section{Lower Bound}\label{sec:lower}
For $k \in \mathbb{N}$, the function class $\mcC_k$ is a class of Boolean functions $g_V:\bits^d\to \bits$ that depend only on the variables in set $V \in \mcV$, where $ \mcV = \{V\subseteq [d]\mid |V|\leq k\}$.  Namely,
$$\mcC_k=\{g_V:\bits^d\to \bits : V\in \mcV\}.$$
$\mcC_k$ could for example be the class of parity functions over at most $k$ variables, that is, $g_V(x) = \prod_{i\in V} x_i$ when $|V| \leq k$. 

We can describe the set $V$ using a string $\tilde{V}$ in $\{\pm 1\}^{k(\log(d)+1)}$. More specifically, we encode every coordinate in $V$ into a $\bits$ string of length $\log(d)$ and use the extra first (most significant) bit to denote that this is a valid coordinate by setting it to $+1$. If $|V| \leq k$, then we encode each of the remaining $k-|V|$ coordinates into a string of length $\log(d)+1$ that consists only of $-1$. Since the binary string $\tilde{V}$ uniquely identifies the corresponding $g_V$, the description length of $\mcC_k$ is $\mathrm{len}(\mcC_k)=k(\log(d)+1)$.

For our lower bound, we consider that set $V$ is a secret that we want to share using a simple secret-sharing scheme. A $t$-secret sharing scheme ``hides" the secret in $t$ shares so that:
\begin{enumerate}
    \item the secret can be reconstructed using the $t$ shares,
    \item the secret cannot be reconstructed with the knowledge of any $t-1$ or fewer shares.
\end{enumerate}
In addition to the secret $V$, our scheme will receive as input a vector $r \in \bits^{tk(\log(d)+1)}$, which represents the randomness of the scheme. For every $p \in [t]$ we denote the $p^{\text{th}}$ share of the secret-sharing scheme with input secret $V$, randomness vector $r$, and threshold $t$, by $\text{share}^t(V;r)_p$. It is generated as follows:
\[
\text{share}^t(V;r)_p = 
\begin{cases}
\left(r_{(p-1)k(\log(d)+1)+1}, \ldots, r_{p k (\log(d)+1)}\right), & \text{if } p\in [t-1]\\
((\prod_{p'=1}^{t-1} r_{(p-1)k(\log(d)+1)+1})\ \tilde{V}_1,\ldots, (\prod_{p'=1}^{t-1} r_{pk(\log(d)+1)})\ \tilde{V}_{k(\log(d)+1)}), & \text{if } p = t.
\end{cases}
\]
To reconstruct the secret $V$ we first compute $(\prod_{p=1}^t \text{share}^t(V;r)_{p,1},\ldots,\prod_{p=1}^t \text{share}^t(V;r)_{p,k(\log(d)+1)}) $ which must be equal to $\tilde{V}$. Then, we split the string into $k$ substrings of length $\log(d)+1$ and interpret each separately. If a substring starts with $-1$, we ignore it. Otherwise, we ignore the first bit and consider that the element whose binary encoding we see is in $V$.

Using this secret sharing scheme, we define the class of functions $\mcC_k^{(t)}$ and distribution $E_{\epsilon}$, for which we prove the separation of~\Cref{thm:lwrbnd} as follows:
\begin{definition}\label{def:lbclass}
Given function class $\mcC_k$, we define the class $\mcC_k^{(t)}=\{f_{V,r} \mid V \in \mcV \text{ and } r\in \bits^{tk(\log(d)+1)} \}$ of functions $f_{V,r}:\bits^{d+\log(tk(\log d+1))+1}\to\bits$ defined by
\[
f_{V,r}(x,p,q,b) = 
\begin{cases}
\mathrm{share}^{t}(V;r)_{p,q}, &\text{if }b=1\\
g_V(x), &\text{if }b=-1,
\end{cases}
\]
where $g_V \in \mcC_k$, $q \in \bits^{\log(k(\log d+1))}$ and $p \in \bits^{\log(t)}$ are used to indicate the $q^\text{th}$ bit of the $p^\text{th}$ share.
\end{definition}

\begin{definition}\label{def:lbdistribution}
We let distribution $E_{\varepsilon}=\mathrm{Ber}(1/2)^{\log(tk(\log(d)+1))}\times \mathrm{Ber}(\varepsilon)$, where $\mathrm{Ber}(\varepsilon)$ denotes the Bernoulli distribution with support $\bits$ and parameter $\varepsilon$.
\end{definition}
 We draw $(p,q,b)$ from $E_\varepsilon$. That is, $p$ and $q$ are chosen uniformly at random, whereas the bit $b=1$ with probability $\varepsilon$ and $-1$ otherwise. 
\begin{theorem}
Function class $\mcC_k^{(t)}$ (\Cref{def:lbclass})
is attribute-efficient learnable for the class of distributions $D'$ over labeled examples, where each example's features are drawn from $D_x\times E_\epsilon$ s.t. $D_x \in \Delta(\bits^d)$ and $E_\varepsilon$ as in~\Cref{def:lbdistribution}, with accuracy parameters $(\varepsilon,\delta)$, sample complexity $N = 
\tilde{O}\left(tk\log(d)\frac{\log(1/\delta)}{\varepsilon}\right)=\tilde{O}(\mathrm{len}(\mcC^{(t)}_k))$, and time complexity $O(N)$.
\label{thm:ck'ael}
\end{theorem}

\begin{proof}
Let $S$ be the input sample of size $N = \frac{8}{\varepsilon}tk(\log(d)+1)\ln(\frac{2tk(\log(d)+1)}{\delta})$. Each labeled example from $S$ is of the form $((x,p,q,b),f_{V,r}(x,p,q,b))$ where $x\sim D_x$, $p$ and $q$ are drawn uniformly at random, and $b\sim \mathrm{Ber}(\varepsilon)$.
By the definition of the secret-sharing scheme, since $V,r$ are fixed, given all the bits of all the shares, that is, if $\forall p,q$, there exists $x\in \bits^d$ such that example $(x,p,q,1)\in S$, then we can recover the secret $V$. We will describe the recovery algorithm at the end of this proof.

We denote the number of examples with $b=1$ by $N_1$. Let $Y_1, \ldots, Y_{N_1}$ be independent random variables such that $Y_i = \Ind\{b_i=1\}$, where $b_i$ is the value of $b$ of the $i^{\text{th}}$ example and $\Ind\{A\}$ is $1$ when $A$ is true and $0$ otherwise. 
Then $N_1 = \sum_{i=1}^{N}Y_i$ and $\E[N_1] = \varepsilon N$.
Let $H_{p,q}$ be the event that $\nexists x : (x,p,q,1)\in S$, that is, the dataset does not include an example of the form $(x,p,q,1)$ for any $x\in\bits^d$.

\begin{equation}\label{eq:missingshare}
  \Pr\left[\exists p,q ~ H_{p,q}\right] \leq \Pr\left[\exists p,q ~ H_{p,q} | N_1 \geq \frac{\varepsilon}{8} N\right]+
  \Pr\left[N_1 < \frac{\varepsilon}{8} N\right]
\end{equation}
We first bound the first term of \Cref{eq:missingshare}.
Since $p,q$ are drawn uniformly at random, the probability that a pair $(p,q)$ is drawn is $\left(2^{\log(t)+\log(k(\log(d)+1))}\right)^{-1}=(tk(\log(d)+1))^{-1}$. 
By union bound we have that
\begin{align*}
    \Pr\left[\exists p,q ~ H_{p,q} | N_1 \geq \frac{\varepsilon}{8} N\right]
    & \leq tk(\log(d)+1) \Pr\left[H_{p,q} | N_1 \geq \frac{\varepsilon}{8} N\right]\\
    & \leq tk(\log(d)+1)\left(1-\frac{1}{tk(\log(d)+1)}\right)^{\frac{\varepsilon}{8} N}\\
    & = tk(\log(d)+1)\left(1-\frac{1}{tk(\log(d)+1)}\right)^{tk(\log(d)+1)\ln(2tk(\log(d)+1)/\delta)} \tag{substituting for the value of $\frac{\varepsilon}{8} N$}\\
    & \leq tk(\log(d)+1) e^{-\ln(2tk(\log(d)+1)/\delta)} \tag{since $(1-1/x)^x\leq e^{-x}$} \\
    & \leq \frac{\delta}{2}
\end{align*}

We now turn to the second term of \Cref{eq:missingshare}.
Since $N>\frac{8}{\varepsilon}\ln\frac{2}{\delta}$, it is easy to verify that $\frac{\varepsilon}{2}N> \sqrt{2\varepsilon N \ln\frac{2}{\delta}}$. This implies that $\frac{\varepsilon}{8} N<\frac{\varepsilon}{2} N< \varepsilon N -\sqrt{2\varepsilon N \ln(2/\delta)}$.

Recall that $N_1 = \sum_{i=1}^{N}Y_i$ is a sum of independent random variables in $\{0,1\}$ and $\E[N_1] = \varepsilon N$. Let $\beta=\sqrt{\frac{2\ln(2/\delta)}{\varepsilon N}}\in(0,1)$, then
\begin{align*}
    \Pr\left[N_1<\frac{\varepsilon N}{8}\right]
    & \leq \Pr\left[N_1<\varepsilon N -\sqrt{2\varepsilon N \ln(2/\delta)}\right]\\
    & = \Pr\left[N_1<(1-\beta)\varepsilon N\right]\\
    & \leq e^{-\beta^2\varepsilon N/2} \tag{by Chernoff bounds~\cite[Theorem 4.5]{MU17}}\\
    & = \frac{\delta}{2}.
\end{align*}
Overall, by \Cref{eq:missingshare}, we conclude that with probability at least $1-\delta $, the input sample includes all pairs $p,q$. 

The process of reconstructing the secret set $V$ computes the product of the labels of the samples with $b=1$ per coordinate $q$ for all $p$, ignoring possible duplicates, which takes time
$O(N)=O\left(\frac{tk(\log(d)+1)}{\varepsilon}\ln\left(\frac{2tk(\log(d)+1)}{\delta}\right)\right)$. 
Then converts $\tilde{V}$ to $V$ in $O(k(\log(d)+1)$ time. After recovering the secret $V$, our algorithm returns the unique function $g_V$. Overall, the algorithm has time complexity $O\left(\frac{tk\log(d)}{\varepsilon}\ln\left(\frac{tk\log(d)}{\delta}\right)\right)$. 
The description length of the class $\mcC_k^{(t)}$ is $\mathrm{len}(\mcC_k^{(t)})=\log(|\mcC_k^{(t)}|)=\log(|\mcV|\cdot 2^{tk(\log (d)+1)}) = O(tk\log(d))$.
Thus, given polynomial in the description length of class $\mcC_k^{(t)}$ examples, this algorithm runs in polynomial time in $\mathrm{len}(\mcC_k^{(t)})$, and with probability $1-\delta$, returns a hypothesis $h=g_V$ with error at most $\varepsilon$, that is,
\begin{align*}
 & \Prx_{(x,p,q,b)\sim D_x\times E_\varepsilon}\left[h(x,p,q,b)\neq f_{V,r}(x,p,q,b)\right] \\
 \leq{} & \Prx_{(x,p,q)\sim D_x\times E_\varepsilon^{(1)}}\left[h(x,p,q,1)\neq f_{V,r}(x,p,q,1)\right]\cdot \varepsilon \\
 &+ \Prx_{(x,p,q)\sim D_x\times E_\varepsilon^{(-1)}}\left[h(x,p,q,-1)\neq f_{V,r}(x,p,q,-1)\right] \\
 \leq{} & \varepsilon,
\end{align*}
where $E_\varepsilon^{(1)}$ and $E_\varepsilon^{(-1)}$ correspond to distribution $E_\varepsilon$ conditioned on $b=1$ and $b=-1$, respectively.

The parameter $t$ is a free parameter that determines the description length of $\mcC_k^{(t)}$ and the time and sample complexity correspondingly. However, for any $t\in\N$, $\mcC_k^{(t)}$ is attribute-efficient learnable with sample complexity $\mathrm{poly}(\mathrm{len}(\mcC_k^{(t)}))$ and time complexity $\mathrm{poly}(\mathrm{len}(\mcC_k^{(t)}))$, that is, polynomial in the size of the input dataset.
\end{proof}
\begin{theorem}
Assume that at least $s\leq \mathrm{poly}(d)$ examples are necessary to learn class $\mcC_k$ in $\mathrm{poly}(d)$ time for distribution $D$ over labeled examples, where the features are drawn from $D_x \in \Delta(\bits^d)$, and for all accuracy parameters $\varepsilon\in (0,1/2)$, $\delta\in(0,1)$. 
Let set of subsets $\mcV_k^{(t)}=\{V=V_x\cup \{d+1,\ldots,d+\log(tk(\log d+1))+1\} \mid V_x \subseteq [d], |V_x|\leq k\}$.
Then for distribution $D'$, where the features are drawn from $D_x\times E_\varepsilon$, and $n$ tasks with $m \geq 
O\left(\frac{k}{\varepsilon}\log(\frac{1}{\varepsilon})\log(\frac{n}{\delta})\right)$ samples-per-task,
such that 
$O\left(\frac{1}{\epsilon}(k\log(d)+\log(\frac{1}{\delta}))\right) \leq nm <s$, 
class $\mcC_k^{(m+1)}$
is $\mcV_k^{(m+1)}$-multitask learnable in $(2d)^k\mathrm{poly}(d)$ time with accuracy parameters $(2\varepsilon, \delta)$, but  $\mcC_k^{(m+1)}$ is not $\mcV_k^{(m+1)}$-multitask learnable in $\mathrm{poly}(d)$ time with accuracy parameters $(\frac{\varepsilon}{32},\delta)$.
\label{thm:lwrbnd}
\end{theorem}
\begin{proof}
We first prove that $\mcC_k^{(m+1)}$ is $\mcV_k^{(m+1)}$-multitask learnable in exponential time in $k\log(d)$. We consider that the given dataset is $S = \{(x_j^{(i)}, p_j^{(i)}, q_j^{(i)}, b_j^{(i)}, y_j^{(i)})_{i\in [n], j \in [m]}\}$, where $(x_j^{(i)}, p_j^{(i)}, q_j^{(i)}, b_j^{(i)}, y_j^{(i)})$ is drawn i.i.d. from $D^{(i)}$, that is $ (x_j^{(i)}, p_j^{(i)}, q_j^{(i)}, b_j^{(i)}) \sim D_x \times E_\varepsilon$ and $y_j^{(i)} = f^{(i)}(x_j^{(i)},p_j^{(i)},q_j^{(i)},b_j^{(i)})$, for $f^{(i)}\in \mcC_k^{(t)}$.

Every task $i$ is associated with a set $V^{(i)}\subseteq V_x$ such that $f^{(i)}(x,p,q,-1) = g_{V^{(i)}}(x)$. The naive algorithm iterates over every possible $V_x \subseteq [d]$ of size $k$ and for each task $i \in [n]$ finds a $\hat{V}^{(i)}\subseteq V_x$ that defines a function $g_{\hat{V}^{(i)}} \in \mcC_k$ that is consistent with the samples of this task with $b = -1$. 
In more detail, in time $O(\binom{d}{k}nm2^k) \leq O((2d)^knm)$, this algorithm finds $\hat{h}^{(1)}, \ldots, \hat{h}^{(n)}$ in $\mcC_k$ such that
\[
 \hat{h}^{(i)} (x_j^{(i)}, p_j^{(i)}, q_j^{(i)}, b_j^{(i)})= y_j^{(i)}, \forall i \in [n] \text{ and } j \in [m] \text{ s.t. } b_j^{(i)} = -1.
\]

 We will now show that for any $f^{(1)},\ldots, f^{(n)}\in \mcC_k^{(t)}$ and $D^{(1)},\ldots, D^{(n)}$ where the features are drawn from $D_x\times E_\varepsilon$, with probability at least $1-\delta$ over the samples the average error of the functions $\hat{h}^{(1)},\ldots, \hat{h}^{(n)}$ is
\[
\frac{1}{n}\sum_{i\in [n]} \Prx_{(x,p,q,b,y)\sim D^{(i)}}[\hat{h}^{(i)}(x,p,q,b)\neq y] \leq 2\epsilon.
\]
There are two ways the result is influenced by our decision to ignore samples with $b = 1$. Firstly, when we draw a new sample to predict its label we have not learnt anything about the secret shares, thus
\begin{align*}
&\frac{1}{n}\sum_{i\in [n]} \Prx_{(x,p,q,b,y) \sim D^{(i)}}[\hat{h}^{(i)}(x,p,q,b)\neq y] \leq \\
&\frac{1}{n}\sum_{i\in [n]} \Prx_{(x,p,q,b,y) \sim D^{(i)}}[\hat{h}^{(i)}(x,p,q,b)\neq y\mid b=-1] +\epsilon.\tag{$\Prx[b=1] = \varepsilon$}
\end{align*}
Secondly, the number of samples we actually use is smaller than the number of samples we have in total. Let $L^{(i)}(h)=\Prx_{(x,p,q,b,y) \sim D^{(i)}}[h(x,p,q,b)\neq y\mid b=-1]$ be the error of function $h$ when it is used for task $i$ . Then, 

\begin{align}\label{eq:prerrorbreakdown}
&\Prx_{S}
\left[\frac{1}{n}\sum_{i \in [n]} L^{(i)}(\hat{h}^{(i)})\geq \epsilon\right] \nonumber\\ 
&\leq \Prx_{S}\left[\frac{1}{n}\sum_{i \in [n]} L^{(i)}(\hat{h}^{(i)})\geq \epsilon \mid \left(N_{-1} >O\left(\frac{k\log(d)+\log(\frac{3}{\delta})}{\varepsilon}\right)\right) \wedge \left(\forall i \in [n] ~~ N_{-1}^{(i)} > O\left(\frac{k}{\varepsilon}\log(\frac{1}{\varepsilon})\right)\right)\right] \nonumber\\
& + \Prx_{S}\left[N_{-1} \leq O\left(\frac{k\log(d)+\log(\frac{3}{\delta})}{\varepsilon}\right)\right] + \Prx_{S}\left[\exists i \in [n]: N_{-1}^{(i)} \leq O\left(\frac{k}{\varepsilon}\log(\frac{1}{\varepsilon})\right)\right],
\end{align}
where $N_{-1}$ is the total number of samples with $b=-1$ and $N_{-1}^{(i)}$ is the number of samples of task $i$ with $b=-1$.

Our approach learns functions $g_{V^{(1)}}, \ldots, g_{V^{(n)}} \in \mcC_k$ using only the relevant samples. Let $\mcV = \{V_x: V_x\subseteq[d], |V_x|\leq k\}$ and  $(\mcC_k\mid V_x) = \{g \in \mcC_k: \text{Rel}(g)\subseteq V_x\}$, then $\text{VC}(\mcC_k | \mcV) = \max_{V_x \in \mcV} \log(|(\mcC_k\mid V_x)|) = \max_{V_x \in \mcV} \log(|\{g_V \in \mcC_k: V \subseteq V_x\}|) = k$. Additionally, we have that $|\mcV| \leq \binom{d}{k}2^k \leq e^{k\ln(ed)}$. Applying \Cref{thm:generalization}, we have that for $N_{-1} \geq O(\frac{k\log(d)+\log(3/\delta)}{\varepsilon})$ and $\min_{i\in [n]}{N_{-1}^{(i)}} \geq O(\frac{k}{\varepsilon}\log(1/\varepsilon))$ with probability at least $1-\delta/3$ over the samples with $b=-1$ we output $\hat{h}^{(1)}, \dots, \hat{h}^{(n)}$ such that
\[
\frac{1}{n}\sum_{i\in [n]} \Prx_{(x,p,q,b,y) \sim D^{(i)} }[\hat{h}^{(i)}(x,p,q,b)\neq y\mid b=-1] \leq \varepsilon.
\]
This bounds the first term of~\Cref{eq:prerrorbreakdown} by $\delta/3$.


 We can write $N_{-1}$ as the sum of independent random variables $Y_j^{(i)} = \Ind\{b_j^{(i)} = -1\}$, specifically $N_{-1} = \sum_{i\in [n]}\sum_{j\in [m]}Y_j^{(i)}$, with expectation $\mathbb{E}[N_{-1}] = (1-\epsilon)nm$.  For $nm > 8 \cdot \frac{1}{(1-\epsilon)\epsilon}(k\log(d)+\log(\frac{3}{\delta}))$, we can see that $\frac{(1-\epsilon)}{2}nm > \sqrt{2(1-\epsilon)nm\ln(3/\delta)}$ and, hence, $\frac{(1-\epsilon)}{8}nm<\frac{(1-\epsilon)}{2}nm < (1-\epsilon)nm-\sqrt{2(1-\epsilon)nm\ln(3/\delta)}$. As a result, for $\beta = \sqrt{\frac{2\ln(3/\delta)}{(1-\epsilon)nm}} \in (0,1)$

\begin{align*}
\Prx\left[N_{-1} \leq O\left(\frac{1}{\epsilon}(k\ln(d)+\ln(\frac{3}{\delta})\right)\right] & \leq \Prx\left[N_{-1} \leq \frac{(1-\epsilon)}{8}nm\right] \\
& = \Prx[N_{-1} \leq (1-\beta)(1-\epsilon)nm] \\
& \leq e^{-\beta^2 (1-\epsilon)nm/2} = \frac{\delta}{3}. \tag{by Chernoff bound~\cite[Theorem 4.5]{MU17}}
\end{align*}
This bounds the second term of~\Cref{eq:prerrorbreakdown} by $\delta/3$, for $nm \geq O(\frac{1}{\varepsilon}(k\log{(d)} + \log{(\frac{1}{\delta})}))$, because $\frac{1}{1-\epsilon}<2$.

 Similarly, we see that for all tasks $i \in [n]$ the number of samples we use can also be written as a sum of random variables, i.e. $N_{-1}^{(i)} =  \sum_{j\in [m]} Y_j^{(i)}$, with expectation  $\mathbb{E}[N_{-1}^{(i)}] = (1-\epsilon)m$. Assuming that $m \geq 8\frac{k}{(1-\varepsilon)\varepsilon}\log(1/\varepsilon)\log(3n/\delta)$,  for all tasks $\frac{(1-\epsilon)}{2}m > \sqrt{2(1-\epsilon)m\ln(3n/\delta)}$ and $\frac{(1-\epsilon)}{8\ln(3n/\delta)}m < \frac{(1-\epsilon)m}{2}< (1-\epsilon)m - \sqrt{2(1-\epsilon)m\ln(3n/\delta)}$.Therefore, for $\gamma = \sqrt{\frac{2\ln(3/\delta)}{(1-\epsilon)m}} \in (0,1)$
\begin{align*}
    \Prx\left[\exists i \in [n]: N_{-1}^{(i)} \leq O\left(\frac{k}{\varepsilon}\log(1/\varepsilon)\right)\right] 
    &\leq n \Prx\left[ N_{-1}^{(i)} \leq O\left(\frac{k}{\varepsilon}\log(1/\varepsilon)\right)\right]\\
    &\leq n\Prx\left[N_{-1}^{(i)} \leq \frac{(1-\epsilon)}{8\ln(3n/\delta)}m\right]\\
    &\leq n\Prx\left[N_{-1}^{(i)} \leq (1-\epsilon)m - \sqrt{2(1-\epsilon)m\ln(3n/\delta)}\right]\\
    &= n\Prx\left[N_{-1}^{(i)} \leq(1-\gamma)(1-\epsilon)m\right]\\
    &\leq e^{-\gamma^2(1-\epsilon)m/2}= \frac{\delta}{3}.\tag{by Chernoff bound~\cite[Theorem 4.5]{MU17}}
\end{align*}
This bounds the last term of~\Cref{eq:prerrorbreakdown} by $\delta/3$ for $m\geq O(\frac{k}{\varepsilon}\log{(1\varepsilon)}\log{(n/\delta)})$.
All in all, we see that with probability at least $1-\delta$ over the samples
\[
\frac{1}{n}\sum_{i\in [n]} \Prx_{(x,p,q,b,y)\sim D^{(i)}}[\hat{h}^{(i)}(x,p,q,b)\neq y\mid b=-1]\leq \epsilon
\]
and, hence, 
\[
\frac{1}{n}\sum_{i\in [n]} \Prx_{(x,p,q,b,y)\sim D^{(i)}}[\hat{h}^{(i)}(x,p,q,b)\neq y]\leq 2\epsilon.
\]
Therefore, we conclude that $\mcC_k^{(m+1)}$ is $\mcV_k^{(m+1)}$-multitask learnable with accuracy parameters $(2\varepsilon,\delta)$, at most $m=O\left(\frac{k}{\varepsilon}\log(\frac{1}{\varepsilon})\log\left(\frac{\log(d)}{\delta}\right)\right)$ samples-per-task and $nm=O\left(\frac{k}{\varepsilon}(\log(d)+\log(\frac{1}{\delta}))\right)$ samples overall, in time $(2d)^knm$. Since $nm < s\leq \mathrm{poly}(d)$, the time is $(2d)^k\mathrm{poly}(d)$. 
This concludes the first part of the proof.

 For the second, let $\mcA$ be a $\mathrm{poly}(d)$-time $\mcV_k^{(m+1)}$-multitask learning algorithm for class $\mcC_k^{(m+1)}$ with accuracy $(\varepsilon/32,\delta)$. For any $f^{(1)},\ldots,f^{(n)}\in \mcC_k^{(m+1)}$ and $D^{(1)},\ldots,D^{(n)} $ where the features of the examples are drawn from $D_x \times E_\varepsilon$, with probability at least $1-\delta$, it returns $h^{(1)},\ldots,h^{(n)}$ such that \begin{equation}\label{eq:mtlguarantee_reduction}
\frac{1}{n}\sum_{i\in [n]}\Prx_{(x,p,q,b,y)\sim D^{(i)}}[h^{(i)}(x,p,q,b)\neq y]\leq \frac{\varepsilon}{32}.
\end{equation}

We will use $\mcA$ to attribute-efficient learn the function class $\mcC_k$.

Let $S = ((x_1,y_1),\ldots,(x_N,y_N))$ be the dataset, where $x_i \sim D_x$, and $y_i = g_V(x_i)$ for $g_V \in \mcC_k$. We construct dataset $S'$ as follows. We first split the dataset into $n$ tasks, with $m$ examples each, denoting the $j^{\text{th}}$ example of the $i^{\text{th}}$ task by $(x_j^{(i)}, y_j^{(i)})$. We choose $\{r_i\}_{i\in [n]}$ from $\bits^{(m+1)k(\log(d)+1)}$, such that for all pairs $i\neq i'$, $r_i\neq r_{i'}$. Then for every example $j\in[m]$ of every task $i\in[n]$ we draw $(p_j^{(i)},q_j^{(i)},b_j^{(i)})$ from distribution $E_\varepsilon$. We set:
\[
\tilde{y}^{(i)}_j = 
\begin{cases}
\text{share}^{(m+1)}(V_\text{aux};r_i)_{p,q}&\text{if } b=1\\
y^{(i)}_j&\text{if } b=-1
\end{cases}
\]
where $V_\text{aux}$ is such that $\tilde{V}_\text{aux} = \{-1\}^{k(\log(d)+1)}$. Creating the new dataset of size $nm \times (d+\log((m+1)k(\log{d}+1))+2)$ requires time $O(nm(m+k\log(d))$, which is at most $s^2+sd\log(d)\leq \mathrm{poly}(d)$. 

For the new dataset  $S' =\{(x_j^{(i)},p_j^{(i)},q_j^{(i)},b_j^{(i)}, \tilde{y}_j^{(i)})_{j\in [m], i\in [n]}\}$ we have that $(x_j^{(i)},p_j^{(i)},q_j^{(i)},b_j^{(i)}) \sim D_x\times E_\varepsilon$. 
Moreover, for every task $i\in[n]$, since the number of examples is $m$, which is smaller than the reconstruction threshold of the secret-sharing scheme which is $m+1$, there exists $r_i'$ such that $\text{share}^{(m+1)}(V_\text{aux};r_i)_p=\text{share}^{(m+1)}(V;r_i')_p$ for all $p$ that appear in this task's dataset. Thus, there exist $f^{(i)}=f_{V,r_i'}$ in $\mcC_k^{(m+1)}$ for every $i\in[n]$, such that $\tilde{y}_j^{(i)} = f^{(i)}(x_j^{(i)},p_j^{(i)},q_j^{(i)},b_j^{(i)})$. 

If we give $\mcA$ the dataset $S'$, then by assumption, in $\mathrm{poly}(d)$ time with probability $1-\delta$, it returns functions $h^{(1)}, \ldots, h^{(n)}$ that satisfy the guarantee of \Cref{eq:mtlguarantee_reduction}.
Since $\Pr[b=-1]=1-\varepsilon$, we have that
\begin{equation}\label{eq:avgerror_reduction}
\frac{1}{n}\sum_{i=1}^n \Prx_{(x,p,q) \sim D_x\times E^{(-1)}_\varepsilon}\left[h^{(i)}(x,p,q,-1)\neq g_V(x)\right] \leq \frac{\varepsilon}{32(1-\varepsilon)}\leq \frac{\varepsilon}{16},
\end{equation}
where $E^{(-1)}_\varepsilon$ is the distribution $E_\varepsilon$ conditional on $b=-1$.

In order to get a prediction for a new $x$, we compute 
\[
h(x) = \text{majority}(\{h^{(i)}(x,p,q,-1)\}_{i\in [n], p\in [m+1], q \in [k(\log(d)+1)]}).
\]
This process takes $O(n(m+1)k(\log(d)+1) d)$ time.
We define the set of good parameters 
\begin{equation}\label{eq:goodset}
\texttt{Good} = \left\{(i,p,q)\in [n]\times [m+1] \times [k(\log(d)+1)]:\Pr_{x\sim D_x}\left[h^{(i)}(x,p,q,-1)\neq g_V(x)\right] \leq \frac{\varepsilon}{4}\right\}.
\end{equation}
The size of $\texttt{Good}$ is at least $\frac{3}{4}n (m+1)k(\log(d)+1)$ because otherwise, if 
\[|\texttt{Bad}|=\left|[n]\times [m+1] \times [k(\log(d)+1)]\setminus \texttt{Good}\right|>\frac{1}{4}n (m+1)k(\log(d)+1),\] by \Cref{eq:avgerror_reduction},
\begin{align*}
&\frac{1}{n (m+1) k(\log(d)+1)}\sum_{i=1}^n\sum_{p=1}^{m+1}\sum_{q=1}^{k(\log(d)+1)}\Prx_{x \sim D_x}\left[h^{(i)}(x,p,q,-1)\neq g_V(x)\right] \leq \frac{\varepsilon}{16}\\
\Rightarrow{} &\frac{1}{n (m+1) k(\log(d)+1)}\sum_{(i,p,q)\in \texttt{Bad}}\Prx_{x \sim D_x}\left[h^{(i)}(x,p,q,-1)\neq g_V(x)\right] \leq \frac{\varepsilon}{16} \\
\Rightarrow{} &\frac{|\texttt{Bad}|}{n (m+1) k(\log(d)+1)} \cdot \frac{\varepsilon}{4} < \frac{\varepsilon}{16} \\
\Rightarrow{} &\frac{1}{4} \cdot \frac{\varepsilon}{4} < \frac{\varepsilon}{16},
\end{align*}
which is a contradiction.

For the majority $h(x)$ to make a mistake, at least half of $\{h^{(i)}(x,p,q,-1)\}_{i\in [n], p\in [m+1], q \in [k(\log(d)+1)]}$ must make a mistake on $x$. This requires that at least a $\frac{1}{4}$ fraction of $\{h^{(i)}(x,p,q,-1)\}_{(i,p,q)\in\texttt{Good}}$ makes a mistake. Thus, by the definition of the \texttt{Good} set (\Cref{eq:goodset}),
\[
\Pr_{x \sim D_x}[h(x)\neq g_V(x)]\leq 4\cdot \frac{\varepsilon}{4} \leq \varepsilon.
\]
Therefore, this process runs in $\mathrm{poly}(d)$ time and returns a hypothesis $h$ such that with probability at least $1-\delta$, has error at most $\varepsilon$ for every function of class $\mcC_k$. By our assumption, this would be a contradiction and as a result we conclude that no such algorithm $\mcA$ exists.
\end{proof}

For threshold $t=m+1$, \Cref{thm:ck'ael} says that for $N \geq \tilde{O}(\frac{mk}{\varepsilon}\log(d)\log(\frac{1}{\delta}))$, samples we can attribute-efficient learn $\mcC_k^{(m+1)}$ for distribution over the features of the examples $D_x\times E_\varepsilon$ with parameters $(\varepsilon,\delta)$. But by \Cref{thm:lwrbnd}, for $nm = N \in [O(\frac{1}{\varepsilon}(k\log{(d)}+\log{(\frac{1}{\delta})})),s)$ samples we cannot $\mcV_k^{(m+1)}$-multitask learn $\mcC_k^{(m+1)}$ for the same distribution in polynomial time with parameters $(\varepsilon/32, \delta)$. Combining the two theorems, we see that we can attribute-efficient learn $\mcC_k^{(m+1)}$ but not multitask learn it in polynomial time with respect to the size of the input, when each task has $m$ samples for $ m \in \left[ O\left(\frac{k}{\varepsilon}\log({\frac{1}{\varepsilon}})\log({\frac{n}{\delta}})\right), \tilde{O}\left( \frac{\varepsilon \cdot s}{ k \log{d}\log{(\frac{1}{\delta})}}\right)\right)$.

\paragraph{Learning $k$-sparse parities.}
To make the result more concrete we look at the case where $\mcC_k$ consists of $k$-sparse parities. The best known polynomial time algorithms to learn this $\mcC_k$ require $\Omega(d^{1-1/k})$ samples~\cite{KlivansS06}. Assuming that this is optimal and the number of samples we need to efficiently learn this class is $s=d^{1-1/k}$, then for $N \in \left[ \tilde{O}(\frac{k^2 \log{(d)}}{\varepsilon^2}\log{(\frac{1}{\delta})}\log{(\frac{n}{\delta})}), d^{1-1/k}\right)$ samples we can multitask learn $\mcC_k^{(m+1)}$ in $(2d)^k \mathrm{poly}(d)$ time for $m = O(\frac{k}{\varepsilon}\log{(\frac{1}{\varepsilon})}\log{(\frac{n}{\delta})})$ and attribute-efficient learn it in $\mathrm{poly}(d)$ time, but we need more samples in total to multitask learn it in $\mathrm{poly}(d)$ time.

\section*{Acknowledgments}
We thank Adam Smith and Zhiwei Steven Wu for helpful discussionss.

\bibliographystyle{alpha}
\bibliography{bib.bib}

\end{document}